\PassOptionsToPackage{prologue,table,dvipsnames}{xcolor}
\documentclass{article} %
\usepackage{iclr2027_conference,times}

\usepackage{amsmath,amsfonts,bm}

\def\eqref#1{equation~\ref{#1}}

\def\1{\bm{1}}

\DeclareMathAlphabet{\mathsfit}{\encodingdefault}{\sfdefault}{m}{sl}
\SetMathAlphabet{\mathsfit}{bold}{\encodingdefault}{\sfdefault}{bx}{n}

\usepackage{mathtools}
\usepackage{enumitem}

\usepackage{xcolor}   %
\usepackage{hyperref}
\usepackage{url}

\definecolor{SelfColor}{rgb}{0.913,0.443,0.196}
\definecolor{UrlColor} {rgb}{0.7098,0.009,0.0}
\definecolor{RefColor} {rgb}{0.082,0.376,0.510}
\hypersetup{
    colorlinks=true,
    linkcolor=UrlColor,   %
    urlcolor=UrlColor,    %
    citecolor=RefColor    %
}

\title{Constant-Curvature Sliced \\Gromov–Wasserstein for Heterogeneous Cross-Curvature Alignment}

\author{Shanglin Li$^{1,2,5\dagger}$, Wenjing Lu$^{3,5\dagger}$, Muyang Li$^{4}$, Nicu Sebe$^{6}$ \& Ziheng Chen$^{6,7}\thanks{Corresponding author.}$    \\
$^{1}$ BIFOLD, Germany, $^{2}$ Technical University of Berlin, Germany\\
$^{3}$ Shanghai Jiao Tong University, China,
$^{4}$ The University of Sydney, Australia\\
$^{5}$ RIKEN AIP, Japan,
$^{6}$ University of Trento, Italy \\
$^{7}$ Max Planck Institute for Intelligent Systems, Germany\\
$^{\dagger}$ Equal contribution\\
}

\usepackage{adjustbox}
\usepackage{booktabs} 
\usepackage{amsmath}
\usepackage{amsthm}
\theoremstyle{plain}
\newtheorem{theorem}{Theorem}[section]

\newtheorem{lemma}[theorem]{Lemma}

\theoremstyle{definition}

\newtheorem{assumption}[theorem]{Assumption}
\theoremstyle{remark}

\usepackage{cleveref}
\crefname{equation}{Eq.}{Eqs.}
\Crefname{equation}{Equation}{Equations}
\crefname{figure}{Fig.}{Figs.}
\Crefname{figure}{Figure}{Figures}
\crefname{table}{Tab.}{Tabs.}
\Crefname{table}{Table}{Tables}
\crefname{section}{Sec.}{Secs.}
\Crefname{section}{Section}{Sections}
\crefname{appendix}{App.}{Apps.}
\Crefname{appendix}{Appendix}{Appendices}

\crefname{theorem}{Thm.}{Thms.}
\Crefname{theorem}{Theorem}{Theorems}
\crefname{lemma}{Lem.}{Lems.}
\Crefname{lemma}{Lemma}{Lemmas}
\crefname{definition}{Def.}{Defs.}
\Crefname{definition}{Definition}{Definitions}
\crefname{corollary}{Cor.}{Cors.}
\Crefname{corollary}{Corollary}{Corollaries}
\crefname{remark}{Rmk.}{Rmks.}
\Crefname{remark}{Remark}{Remarks}
\crefname{proposition}{Prop.}{Props.}
\Crefname{proposition}{Proposition}{Propositions}

\DeclarePairedDelimiterX{\inner}[2]{\langle}{\rangle}{#1,\,#2}

\usepackage{pifont}
\usepackage{algorithm}
\usepackage{algorithmic}

\usepackage{graphicx} %

\usepackage{adjustbox}
\usepackage{color}
\usepackage[textsize=tiny]{todonotes}
\iclrfinalcopy %

\usepackage{subcaption}

\begin{document}

\maketitle

\begin{abstract}
Recent advances in representation learning have highlighted the utility of constant-curvature models, such as hyperbolic and spherical spaces, for modeling complex data. 
Mixed-curvature models further enhance this by integrating multiple constant-curvature components.
However, these models typically learn each component space independently because spaces with different curvatures are inherently heterogeneous and lack a unified metric. 
Consequently, they lack explicit mechanisms to enforce geometric consistency across various spaces. 
Moreover, the problem of comparing probability distributions across mixed-curvature spaces remains unexplored.
To compare distributions on heterogeneous spaces, Gromov–Wasserstein (GW) distances provide a principled framework by aligning their intra-space geometries.
Building on this, we propose constant-curvature sliced Gromov–Wasserstein (CCSGW), a novel divergence for aligning distributions supported on heterogeneous constant-curvature spaces. 
We first introduce the missing geodesic-based one-dimensional projections for spherical spaces, and then extend sliced GW to constant-curvature spaces, enabling efficient and principled comparison across manifolds with different curvatures.
This formulation preserves intrinsic geometric relationships while avoiding the high computational cost. 
We provide theoretical analysis showing that CCSGW controls intrinsic geometric discrepancy across heterogeneous spaces, promoting distribution-level geometric consistency.
By integrating CCSGW into existing mixed-curvature learning tasks, including graph anomaly detection, graph node classification, and multimodal learning, we observe consistent performance gains across diverse settings.
\end{abstract}

\section{Introduction}

Euclidean geometry has long been the standard representation space in machine learning. 
Recently, non-Euclidean spaces, particularly hyperbolic and spherical spaces, have gained attention by providing better representations for certain types of structured data \citep{gu2018learning}.
Hyperbolic space is well-suited to hierarchical organization \citep{khrulkov2020hyperbolic, peng2021hyperbolic,li2026heegnet}, while spherical space effectively captures periodic patterns \citep{cohen2018spherical, meng2019spherical}. 
These spaces can be viewed as special cases of constant-curvature geometries, i.e., geometric spaces whose curvature is uniform everywhere.
Hyperbolic, spherical, and Euclidean spaces are the three types of constant-curvature geometries, corresponding to negative, positive, and zero curvature, respectively \citep{lee2006riemannian}.

Changing the geometry of the underlying latent space can lead to more expressive representations \citep{nickel2017poincare}, with curvature playing a central role. 
However, many real-world datasets do not admit a single geometric structure \citep{peng2021hyperbolic}. 
As a result, any single constant-curvature geometry, Euclidean, hyperbolic, or spherical, may be insufficient to capture all structural nuances simultaneously \citep{chlenskimixed}.
To model real-world data with complex topologies, 
various mixed-curvature models have been proposed that operate across multiple constant-curvature spaces \citep{skopek2020mixed,groverspectro,he2025helm,grover2025curvgad,chen2025galaxy,chlenski2025mixed,gao2025mixed,gu2018learning}.
Mixed-curvature models partition the space into multiple component spaces and learn or fix a curvature for each part, aiming to capture the complex geometric topology \citep{gu2018learning}.

However, these approaches typically learn each component space independently and lack an explicit mechanism to enforce geometric consistency across spaces \citep{gu2018learning}.
Enhancing geometric consistency imposes an additional inductive prior to regularize the representation learning, in cases such as knowledge distillation \citep{park2019relational}, graph matching \citep{xu2019gromov}, and cross-lingual word matching \citep{alvarez2018gromov}.
This mismatch can lead to distortions of the global relational structure, hindering coherent representation learning and potentially degrading downstream performance.
Since data are embedded into heterogeneous metric spaces with different curvatures, distances computed independently within each space are generally not directly comparable.
To address this issue, a principled way to encourage geometric consistency across heterogeneous spaces is needed.

\begin{figure*}[t]
    \centering
    \includegraphics[width=\textwidth]{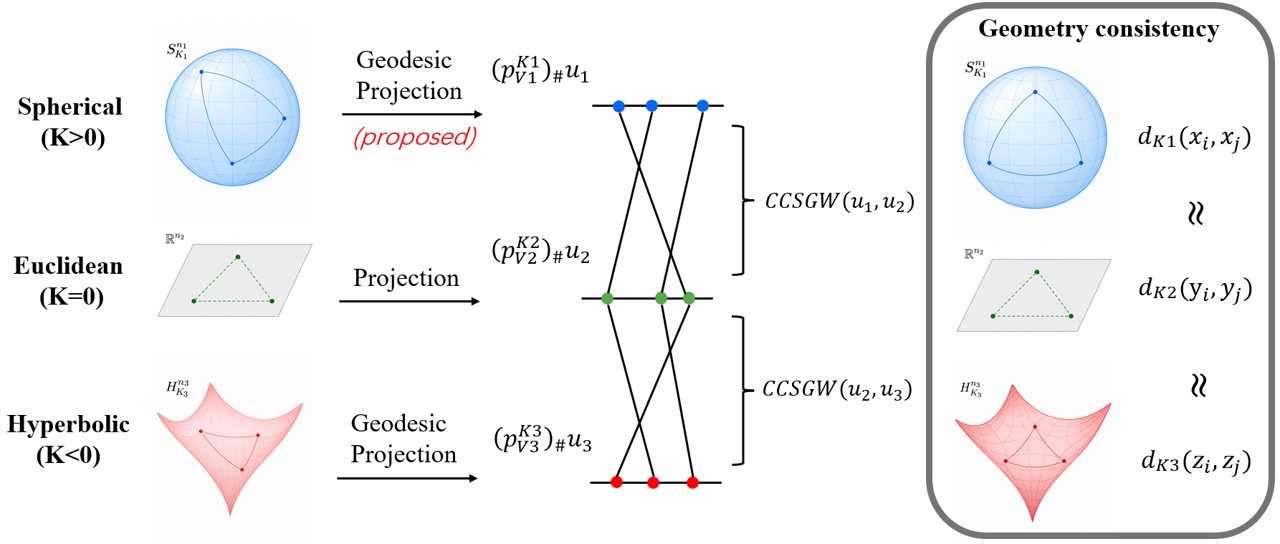}
    \caption{\textbf{CCSGW Overview. } 
    Distributions supported on spherical, Euclidean, and hyperbolic spaces are projected onto one-dimensional coordinates using geometry-aware projections. 
    CCSGW then compares the projected distributions across heterogeneous constant-curvature spaces, encouraging consistency of their intrinsic geometric structures despite differences in curvature.
}
    \label{fig:ccsgw}
\end{figure*}

In this study, we propose constant-curvature sliced-Gromov Wasserstein (CCSGW), a new divergence for comparing distributions on heterogeneous constant-curvature spaces.
CCSGW is an extension of sliced-Gromov–Wasserstein (SGW) \citep{vayer2019sliced} to constant-curvature spaces, allowing explicit alignment across manifolds to enforce geometric consistency.
SGW compares distriptions that do not lie in the same metric space based on their internal pairwise distance structures, while making the comparison computationally efficient by reducing it to a collection of one-dimensional GW problems on the real line.
This property makes SGW particularly well suited to mixed-curvature representations, where components with different curvatures naturally constitute heterogeneous metric spaces and their intrinsic distances are not directly comparable across spaces.
Building on this, we first derive geodesic-based one-dimensional projections for spherical spaces, completing the one-dimensional slicing framework across hyperbolic, Euclidean, and spherical geometries, and then introduce CCSGW.
CCSGW enables meaningful comparison and alignment of distributions across mixed-curvature spaces to preserve distribution-level geometric consistency across mixed-curvature representations (\cref{fig:ccsgw}).
Through theoretical analysis, we show that CCSGW controls the intrinsic Gromov–Wasserstein discrepancy between heterogeneous constant-curvature spaces up to projection error. 
Therefore, minimizing CCSGW encourages approximately distance-preserving alignment for most sample pairs, providing a theoretical justification for distribution-level geometric consistency.
Extensive experiments on simulations and mixed-curvature models for graph anomaly detection \citep{grover2025curvgad}, graph node classification \citep{groverspectro}, and multimodal learning \citep{zhou2026eeg} demonstrate that integrating CCSGW achieves superior performance.

We make three main contributions:
\begin{itemize}[leftmargin=*,itemsep=2pt,topsep=3pt]
    \item We introduce CCSGW, a geometry-aware divergence for aligning
    distributions across heterogeneous constant-curvature spaces, together with
    geodesic-based one-dimensional projections for spherical spaces.

    \item We provide theoretical guarantees connecting CCSGW to intrinsic
    Gromov--Wasserstein discrepancy and approximate pairwise distance preservation.

    \item Experiments on graph anomaly detection, node classification, and
    multimodal learning show consistent improvements when CCSGW is integrated
    into existing mixed-curvature models.
\end{itemize}

\section{Preliminaries}

\subsection{Riemannian manifold}

\textbf{Constant curvature spaces.}
A constant-curvature space can be viewed as a smooth Riemannian manifold with curvature $K \in \mathbb{R}$. The $K$-radius model provides an extrinsic representation of such spaces \citep{skopek2020mixed}, including the sphere $\mathbb{S}_K^n$, Euclidean space $\mathbb{R}^n$, and hyperboloid $\mathbb{H}_K^n$:

\begin{equation}
\label{eq_ccs}
\mathcal{M}_K^n =
\begin{cases}
\mathbb{S}_K^n = \left\{ x \in \mathbb{R}^{n+1} : \|x\|^2 = \frac{1}{K} \right\} 
& \text{for } K > 0,\ \text{spherical geometry}, \\[6pt]
\mathbb{R}^n 
& \text{for } K = 0,\ \text{Euclidean geometry}, \\[6pt]
\mathbb{H}_K^n = \left\{ x \in \mathbb{R}^{n+1} : \|x\|_L^2 = \frac{1}{K},\ x_t > 0 \right\} 
& \text{for } K < 0,\ \text{hyperbolic geometry}.
\end{cases}
\end{equation}
where $\|x\|_L^2 = \|x_s\|^2 - x_t^2$ is the Lorentz inner product.
Following the conventions of the hyperboloid, we write $x = (x_t, x_s^\top)^\top$, with $x_t \in \mathbb{R}$ as the time component and $x_s \in \mathbb{R}^n$ as the spatial component \citep{ratcliffe2006foundations}.

A generalization of straight lines in Euclidean spaces to Riemannian manifolds is given by geodesics, which are smooth curves connecting two points $x, y \in \mathcal{M}$ with minimal length, i.e., continuous mappings $\gamma : [0,1] \to \mathcal{M}$ such that $\gamma(0)=x$, $\gamma(1)=y$.
Following \citep{bonet2023hyperbolic}, we focus on geodesically complete Riemannian manifolds, in which case there is always a geodesic between two points $x, y \in \mathcal{M}$. 
Furthermore, in this specific case, all geodesics are actually geodesic lines, i.e., they can be extended to $\mathbb{R}$. 
Let $x, y \in \mathcal{M}$, and $\gamma : [0,1] \to \mathcal{M}$ be a geodesic between $x$ and $y$ such that $\gamma(0)=x$ and $\gamma(1)=y$.
$\operatorname{PT}_{x \rightarrow y}$ denotes the parallel transport used to transport tangent vectors from the tangent space at $x$ to the tangent space at $y$. 
 
\textbf{Mixed-curvature spaces.}
A mixed-curvature space is defined as the Cartesian product of constant-curvature manifolds \citep{skopek2020mixed}. 
Let each component $\mathcal{M}_{K_j}^{n_j}$ be a space of constant curvature $K_j$ and dimension $d_j$. 
The product space is
\begin{equation}
\mathcal{M}
=
\prod_{j=1}^{m} \mathcal{M}_{K_j}^{\,n_j}.
\end{equation}
For a point $u = (u_1, \dots, u_m) \in \mathcal{M}$, tangent spaces and maps decompose component wise:
\begin{equation}
\begin{aligned}
T_u \mathcal{M}
&=
\prod_{j=1}^{m} T_{u_j} \mathcal{M}_{K_j}^{\,n_j}, \\
\operatorname{Exp}^{\mathcal{M}}_{u}(q)
&=
\bigl(
\exp^{K_1}_{u_1}(q_1), \dots, \exp^{K_m}_{u_m}(q_m)
\bigr), \\
\operatorname{Log}^{\mathcal{M}}_{u}(x)
&=
\bigl(
\log^{K_1}_{u_1}(x_1), \dots, \log^{K_m}_{u_m}(x_m)
\bigr).
\end{aligned}
\end{equation}

\subsection{The sliced-Wasserstein}

The Wasserstein distance is a powerful tool to compare probability distributions, but its computational cost can become prohibitive in high dimensions. 
A key observation is that in one dimension, the Wasserstein distance admits a closed-form solution that can be computed efficiently by sorting samples, leading to a complexity of $O(n \log n)$.
The sliced Wasserstein (SW) distance leverages this property by projecting high-dimensional distributions onto one-dimensional subspaces. 
In practice, the integral is approximated using a finite number of random projections, leading to an efficient approximation with complexity $O(L\, n \log n)$, where $L$ is the number of sampled directions.

\textbf{The Sliced-Gromov Wasserstein.}
We first recall the definition of the Gromov–Wasserstein (GW) distance. 
Let $\mu \in \mathcal{P}(\mathbb{R}^p)$ and $\nu \in \mathcal{P}(\mathbb{R}^q)$ be two probability measures. 
The squared GW discrepancy is defined as
\begin{equation}
\label{eq:gw_recall}
GW_2^2(d^2, \mu, \nu) := \min_{\pi \in \Pi(\mu, \nu)} 
\int |d(x, x') - d(y, y')|^2 \, d\pi(x,y)\, d\pi(x',y'),
\end{equation}
where $\Pi(\mu, \nu)$ denotes the set of couplings between $\mu$ and $\nu$, and $d$ is the Euclidean distance. 
Unlike classical optimal transport, which compares points across domains, the GW distance aligns distributions by matching their intrinsic geometries. More precisely, it seeks a coupling $\pi$ such that pairwise distances within $\mu$ are preserved when mapped to $\nu$, i.e., pairs of points $(x, x')$ in the source are matched to pairs $(y, y')$ in the target with similar distances. This allows comparing distributions supported on different spaces or of different dimensions.

The computational cost of \eqref{eq:gw_recall} is prohibitive for large-scale problems ($O(n^3)$ in the discrete case). Inspired by the sliced Wasserstein (SW) distance, \citep{vayer2019sliced} proposed the SGW discrepancy.
Let $\mathbb{S}^{q-1}$ denote the unit sphere equipped with the uniform measure $\lambda_{q-1}$. 
Let $\mu \in \mathcal{P}(\mathbb{R}^p)$ and 
$\nu \in \mathcal{P}(\mathbb{R}^q)$, and let 
$\Delta \in \mathbb{R}^{q \times p}$ be a linear map. 
The SGW divergence is defined as
\begin{equation}
\label{eq:sgw}
\mathrm{SGW}_{\Delta}(\mu,\nu)
:=
\int_{\mathbb{S}^{q-1}}
\mathrm{GW}_2^2
\left(
d^2,\,
(P_{\theta})_{\sharp}\mu_{\Delta},\,
(P_{\theta})_{\sharp}\nu
\right)
\, d\lambda_{q-1}(\theta),
\end{equation}
where $P_{\theta}(x):=\langle x,\theta\rangle$ denotes the 
one-dimensional projection along $\theta \in \mathbb{S}^{q-1}$.
Here, the inner GW is computed in one dimension, where it admits a closed-form solution based on sorting and evaluating only two candidate permutations (identity or anti-identity) \citep{vayer2019sliced}. This leads to a computational cost of $O(n \log n)$ per projection. In practice, the integral in \eqref{eq:sgw} is approximated by averaging over $L$ random directions, yielding an overall complexity of $O(L\, n \log n)$.

\textbf{Slicing on geodesics.}
The Wasserstein distance can be naturally defined on a Riemannian manifold by replacing the Euclidean distance with the geodesic distance \citep{bonet2023hyperbolic}. 
More precisely, given a manifold $(\mathcal{M}, d)$, the optimal transport cost 
between two probability measures is defined using $d(x,y)$, the length of the 
shortest geodesic connecting $x$ and $y$. This provides a natural extension of 
Wasserstein distances to curved spaces.

To generalize slicing to hyperbolic spaces, points are projected onto geodesics, which play the role of one-dimensional subspaces \citep{bonet2023hyperbolic}. 
Given a direction $v$, let $\gamma_v$ denote the geodesic passing through the origin. The projection of a point $x$ onto $\gamma_v$ is defined as
\begin{equation}
P_v(x) := \arg\min_{y \in \gamma_v} d(x, y),
\end{equation}
and its one-dimensional coordinate is given by the signed geodesic distance to the origin:
\begin{equation}
\label{eq_1d_h}
t_v(x) := \operatorname{sign}(\langle x, v \rangle)\, d(x, x_0),
\end{equation}
where $x_0$ denotes the origin. This mapping yields a scalar representation on $\mathbb{R}$ for each point.

\section{Methods}

\subsection{Sphere geodesic projections}
Recent work has extended the sliced-Wasserstein framework to the sphere by projecting measures onto great circles and computing optimal transport directly on the circle \citep{bonet2023spherical}.
In particular, spherical sliced-Wasserstein relies on closed-form formulations of Wasserstein distances on $\mathcal{S}^1$, which require solving a shift-invariant optimal transport problem due to the periodic structure of the circle.
In our work, we aim to project points onto geodesics and further map them to $\mathbb{R}$ via signed geodesic distances. 
This reduces the problem to standard one-dimensional Wasserstein distances, avoiding the need for circular optimal transport and shift alignment.

To extend slicing to spherical spaces, we derive geodesic projection onto a fixed geodesic on the Sphere (Radius Model).
Let $K > 0$ and consider the sphere (radius model) $S_K^n = \{ x \in \mathbb{R}^{n+1} : \|x\|^2 = 1 / K \}$
For a fixed reference point $x_0 \in S_K^n$ and a fixed unit tangent vector $v \in T_{x_0} S_K^n$ with $\|v\| = 1$.
The geodesic starting from $x_0$ in direction $v$ is given by the exponential map
\begin{equation}
    \gamma(t) = \mathrm{Exp}_{x_0}(t v) = \cos(\sqrt{K}\, t)\, x_0 + \frac{\sin(\sqrt{K}\, t)}{\sqrt{K}}\, v,\qquad t \in \mathbb{R}.
\end{equation}

For any $x \in S_K^n$, we define the geodesic projection parameter $t^\star(x) \in \operatorname*{arg\,min}_{t \in \mathbb{R}} d(x, \gamma(t))$, where the geodesic distance on $S_K^n$ is $d(x,y) = \frac{1}{\sqrt{K}}\,\arccos\!\big( K \langle x, y \rangle \big)$.
Let $x \in S_K^n$, the optimal projection parameter $t^\star(x)$ satisfies the following proposition, which gives a closed-form solution.

\textbf{Proposition 1.}
For $x \in \mathcal{S}^n_K$ with $\langle x, x_0 \rangle \neq 0$, the one-dimensional projection coordinate is given by
\begin{equation}
\label{eq_1d_s}
t_v^*(x) = \frac{1}{\sqrt{K}} \operatorname{atan2}
\left( \langle x, v \rangle, \sqrt{K}\langle x, x_0 \rangle \right).
\end{equation}

The proof and equatorial and degenerate cases can be handled similarly and are provided in Appendix \cref{appendix Sphere geodesic projections}.

\subsection{Constant-curvature sliced Gromov--Wasserstein}
We now define a sliced Gromov--Wasserstein discrepancy for distributions supported on constant-curvature spaces with possibly different curvatures.
Although the heterogeneous spaces and their metrics differ through the curvature parameter $K$, each distribution can be projected onto a one-dimensional geodesic coordinate, yielding a probability measure on the line $\mathbb{R}$.

Let $\mu \in \mathcal{P}(\mathcal{M}_{K_1}^{n_1})$ and $\nu \in \mathcal{P}(\mathcal{M}_{K_2}^{n_2})$, with $n_1 \leq n_2$. At the origins, $T_{o_{K_j}}\mathcal{M}_{K_j}^{n_j}\cong\mathbb{R}^{n_j}$ for $j=1,2$. 
Following \citep{vayer2019sliced}, let $\Delta:\mathbb{R}^{n_1}\to\mathbb{R}^{n_2}$ be the zero-padding map $\Delta(u)=(u,0_{n_2-n_1})$. For $\theta\in\mathbb{S}^{n_2-1}$, let $p_\theta^{K_j}:\mathcal{M}_{K_j}^{n_2}\to\mathbb{R}$ be the geodesic projection coordinate associated with the unit tangent vector represented by $\theta$ at the origin. 
Let $\lambda_{n_2-1}$ be the uniform probability measure on $\mathbb{S}^{n_2-1}$. 
We define
\begin{equation}
\mathrm{CCSGW}_{\Delta}(\mu, \nu)
=
\int_{v \sim \lambda_{K_1}}
\mathrm{GW}_2^2\!\left(
d_{\mathbb{R}}^2,\,
(p^{K_1}_{v})_\# \mu,\,
(p^{K_2}_{\Delta v})_\# \nu
\right)
\, d\lambda_{n_1}(v),
\label{eq:ccsgw}
\end{equation}
Here, $d_{\mathbb{R}}(s,t)=|s-t|$, and both measures are sliced using the same direction $\theta$ after dimension alignment. When $K_1=K_2=0$, this definition reduces to Euclidean SGW. When $n_1=n_2$, the zero-padding map and its induced embedding are the identity.

\subsection{Theoretical Analysis}
\label{sec:theory}
 
We next connect the sliced objective in \cref{eq:ccsgw} to the intrinsic
geometry of the original constant-curvature spaces. We show that CCSGW
controls the intrinsic GW discrepancy up to an explicit error induced by
geodesic projection. Since CCSGW averages discrepancies over
one-dimensional projections, the resulting guarantee is naturally an
average one: a cross-space map approximately preserves pairwise distances
for most pairs, rather than uniformly over all pairs.
 
Let $\mathcal{X}=(X,d_X,\mu)$ and $\mathcal{Y}=(Y,d_Y,\nu)$ be two
constant-curvature metric-measure spaces, with origins $o_X$ and $o_Y$. As in \cref{eq:ccsgw},
directions $v\sim\lambda_X$ and $w\sim\lambda_Y$ are drawn independently
and uniformly from the unit tangent spheres at the origins, and
$p_v^X, p_w^Y$ denote the one-dimensional geodesic projection
coordinates. Throughout, $\operatorname{GW}_2$ denotes the square root of
the discrepancy in \cref{eq:gw_recall} evaluated with the intrinsic geodesic
distances,
$\operatorname{GW}_2(\mathcal{X},\mathcal{Y})
=\inf_{\pi\in\Pi(\mu,\nu)}
\|d_X(x,x')-d_Y(y,y')\|_{L_2(\pi\otimes\pi)}$,
so that
$S:=\operatorname{CCSGW}(\mu,\nu)
=\mathbb{E}_{v,w}\operatorname{GW}_2^2\big((p_v^X)_\#\mu,(p_w^Y)_\#\nu\big)$.
 
We assume that the distributions are supported within a finite radius $R$
of the corresponding origins and, for a spherical component with curvature
$K>0$, that
$R<\frac{\pi}{2\sqrt{K}}$.
Under this condition the geodesic projection is unique and its signed
coordinate is an isometry onto the real line: the difference of two
coordinates equals the geodesic distance between the projected points. Without it, coordinates on a great circle can wrap
around, and the sliced objective no longer reflects the spherical geometry. For $K\le0$ geodesics do not wrap, and only bounded
support is required. Let $P_v^X(x)$ denote the projection of $x$ onto the
geodesic through $o_X$ in direction $v$, and define the average projection
displacement
\begin{equation}
    H_X
    :=
    \left(
    \mathbb{E}_{v}
    \mathbb{E}_{x\sim\mu}
    d_X\!\left(x,P_v^X(x)\right)^2
    \right)^{1/2},
\end{equation}
with $H_Y$ defined analogously using $P_w^Y$. Intuitively, $H_X$ and $H_Y$ quantify how much intrinsic geometric information may be lost by reducing the original spaces to one-dimensional geodesic slices, i.e. the information loss of slice projection.
 
\begin{theorem}
\label{thm:ccsgw_gw}
Let $S := \operatorname{CCSGW}(\mu,\nu)$. Under the above 
conditions,
\[
\operatorname{GW}_2(\mathcal{X},\mathcal{Y})
\le
2H_X + \sqrt{S} + 2H_Y.
\]
\end{theorem}
 
\cref{thm:ccsgw_gw} suggests that a small CCSGW value certifies a small intrinsic GW discrepancy whenever the projection displacements $H_X$ and $H_Y$ are small. These two terms are properties of the data and the projection family; they are explicit and measurable, but they are not reduced by minimizing $S$ alone. The proof lifts, for each pair of directions, the optimal matching between the two one-dimensional slices to a matching between the original spaces, and bounds its distortion by the triangle inequality: the intrinsic discrepancy is at most the distortion introduced by projecting $\mathcal{X}$, plus the one-dimensional GW discrepancy computed by CCSGW, plus the distortion introduced by projecting $\mathcal{Y}$. Averaging over directions yields \cref{thm:ccsgw_gw}.
 
We further translate this distribution-level guarantee into a pairwise distance-preservation result, the average analogue of a Lipschitz bound used in~\citep{li2026rethinking}. Let $\pi^\star\in\Pi(\mu,\nu)$ attain $\operatorname{GW}_2(\mathcal{X},\mathcal{Y})$, let $S_X$ be the $X$-projection of
$\operatorname{supp}\pi^\star$, and assume
\[
r
:=
\inf_{\substack{x,x'\in S_X\\x\neq x'}}
d_X(x,x')
>0,
\]
which holds, e.g., for the empirical measures used in training. For a
measurable cross-space map $g:X\rightarrow Y$, define its $L_2$ matching
error with respect to $\pi^\star$ as
\[
\rho_{\pi^\star}(g)
:=
\left(
\int d_Y(g(x),y)^2\,d\pi^\star(x,y)
\right)^{1/2}.
\]
 
\begin{theorem}
\label{thm:average_isometry}
Under the same assumptions, let
$B:=2H_X+\sqrt{S}+2H_Y+2\rho_{\pi^\star}(g)$,
\[
\left(
\int_{x\neq x'}
\left|
\frac{d_Y(g(x),g(x'))}{d_X(x,x')}-1
\right|^2
d\pi^\star(x,y)\,d\pi^\star(x',y')
\right)^{1/2}
\le
\frac{B}{r}.
\]
Consequently, for every $\epsilon\in(0,1)$, all pairs with $x\neq x'$
except a set of $\pi^\star\otimes\pi^\star$-mass at most $\epsilon$
satisfy
\[
1-\frac{B}{r\sqrt{\epsilon}}
\le
\frac{d_Y(g(x),g(x'))}{d_X(x,x')}
\le
1+\frac{B}{r\sqrt{\epsilon}}.
\]
\end{theorem}
 
\cref{thm:ccsgw_gw} and \cref{thm:average_isometry} give a direct
geometric interpretation of the CCSGW objective. A small sliced
discrepancy, together with small geodesic projection displacement and
small matching error, implies that the cross-space map $g$ is
approximately distance-preserving on most sample pairs. Thus CCSGW
enforces distribution-level geometric consistency across heterogeneous
constant-curvature components, while allowing the components to retain
different curvatures. The guarantee characterizes what a small objective
certifies and is informative in the near-isometric regime $B<r$, later in experiment sections we will showcase the effectiveness of enforcing such consistency from empirical perspective. 

\section{Experiments}

We conduct experiments to evaluate the effectiveness of CCSGW in promoting geometric consistency across heterogeneous constant-curvature spaces.
We first perform a controlled simulation to visualize the cross-curvature alignment induced by CCSGW.
Mixed-curvature models have achieved state-of-the-art performance across diverse tasks, including graph anomaly detection~\citep{grover2025curvgad}, graph node classification~\citep{groverspectro}, and multimodal learning~\citep{zhou2026eeg}.
Building on these strong backbones, we incorporate CCSGW as an additional alignment objective in the training to explicitly enforce geometric consistency across heterogeneous component spaces.

\subsection{Simulations}

We first conduct a synthetic experiment to provide an intuitive visualization of cross-curvature alignment with CCSGW. 
Since the deformation of point configurations is more directly visualized in Euclidean space, we optimize a set of points in Euclidean space to align with a fixed target distribution in hyperbolic space using CCSGW as the alignment objective. 
The Euclidean points are initialized along a line, while the target distribution forms a ring-like structure in the hyperbolic space.

\begin{figure*}[t]
    \centering
    \includegraphics[width=\textwidth]{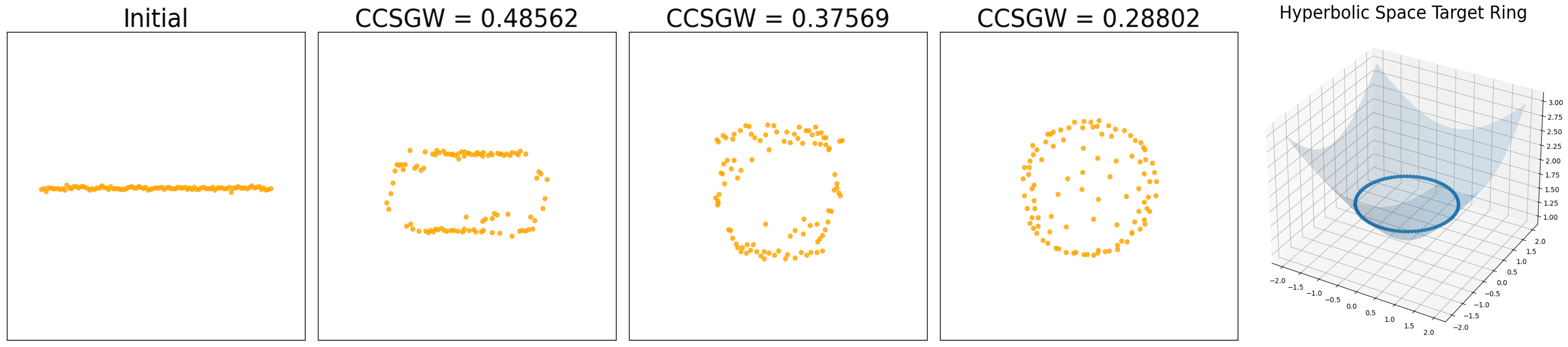}
\caption{
\textbf{Simulation of cross-curvature alignment.}
A Euclidean point cloud initialized as a line is optimized to align with a fixed ring-shaped distribution in hyperbolic space using CCSGW. 
As optimization proceeds, the Euclidean points gradually form a circular structure while the CCSGW discrepancy decreases, illustrating improved geometric consistency.
}
    \label{fig:sim}
\end{figure*}

As optimization proceeds (\cref{fig:sim}), the Euclidean point cloud gradually deforms from the initial line into a circular configuration that better matches the intrinsic relational structure of the hyperbolic target. 
Meanwhile, the CCSGW discrepancy consistently decreases, indicating that the pairwise geometric structures of the two distributions become increasingly aligned despite being supported on spaces with different curvatures. 
This simulation provides a direct illustration that minimizing CCSGW can promote distribution-level geometric consistency across heterogeneous spaces.

\subsection{Main experiments}

\begin{table*}[t]
\caption{Graph anomaly detection performance (AUROC \%($\uparrow$) $\pm$ $95\%$ confidence interval).
\colorbox{red!45}{\textbf{First}}, \colorbox{red!25}{second}, and \colorbox{red!10}{third} best are highlighted.
\textbf{$\Delta$ CCSGW} is the relative gain of CCSGW.}
\normalsize
\centering
\begin{adjustbox}{max width=0.85\textwidth}
\begin{tabular}{lccccc}
\toprule
\textbf{Model} & \textbf{Reddit} & \textbf{Weibo} & \textbf{Amazon} & \textbf{Elliptic} & \textbf{Questions} \\
\midrule

GCN 
& 65.25$\pm$\tiny{0.89} 
& 97.03$\pm$\tiny{0.57} 
& 82.48$\pm$\tiny{0.43} 
& 81.42$\pm$\tiny{1.83} 
& 70.93$\pm$\tiny{1.34} \\

GAT 
& 65.38$\pm$\tiny{0.19} 
& 95.35$\pm$\tiny{1.47} 
& 96.75$\pm$\tiny{1.06} 
& 84.93$\pm$\tiny{1.98} 
& 71.17$\pm$\tiny{1.63} \\

GraphSAGE 
& 62.86$\pm$\tiny{0.14} 
& 94.94$\pm$\tiny{1.37} 
& 89.65$\pm$\tiny{5.42} 
& 85.35$\pm$\tiny{0.72} 
& 72.15$\pm$\tiny{1.73} \\

BernNet 
& 65.82$\pm$\tiny{1.75} 
& 94.96$\pm$\tiny{1.56} 
& 96.21$\pm$\tiny{1.43} 
& 83.18$\pm$\tiny{1.49} 
& 68.82$\pm$\tiny{2.98} \\

ChebyNet 
& 64.11$\pm$\tiny{1.49} 
& 95.02$\pm$\tiny{0.93} 
& 96.59$\pm$\tiny{2.01} 
& 79.34$\pm$\tiny{3.07} 
& 69.23$\pm$\tiny{1.55} \\

\midrule

HGCN 
& 64.13$\pm$\tiny{0.98} 
& 94.23$\pm$\tiny{1.86} 
& 79.94$\pm$\tiny{0.93} 
& 78.66$\pm$\tiny{2.22} 
& 67.16$\pm$\tiny{1.74} \\

HGAT 
& 63.45$\pm$\tiny{0.23} 
& 92.24$\pm$\tiny{0.98} 
& 85.98$\pm$\tiny{3.65} 
& 85.45$\pm$\tiny{0.32} 
& 68.93$\pm$\tiny{1.27} \\

$\kappa$GCN 
& 61.77$\pm$\tiny{1.54} 
& 97.87$\pm$\tiny{1.24} 
& \cellcolor{red!10}97.34$\pm$\tiny{1.19} 
& 76.13$\pm$\tiny{2.13} 
& \cellcolor{red!10}73.15$\pm$\tiny{3.84} \\

$\mathcal{Q}$GCN 
& 66.45$\pm$\tiny{1.12} 
& 97.34$\pm$\tiny{0.34} 
& 96.71$\pm$\tiny{5.93} 
& 85.55$\pm$\tiny{2.03} 
& 71.74$\pm$\tiny{2.26} \\

\midrule

DCI 
& 66.54$\pm$\tiny{3.38} 
& 94.25$\pm$\tiny{1.73} 
& 94.63$\pm$\tiny{0.92} 
& 82.86$\pm$\tiny{1.51} 
& 69.21$\pm$\tiny{1.34} \\

PCGNN 
& 53.25$\pm$\tiny{2.19} 
& 90.27$\pm$\tiny{1.51} 
& 97.31$\pm$\tiny{0.82} 
& \cellcolor{red!10}85.89$\pm$\tiny{1.85} 
& 69.91$\pm$\tiny{1.43} \\

BWGNN 
& 65.43$\pm$\tiny{4.35} 
& 97.38$\pm$\tiny{0.93} 
& 97.05$\pm$\tiny{0.76} 
& 85.29$\pm$\tiny{1.14} 
& 71.85$\pm$\tiny{2.23} \\

DOMINANT 
& 62.32$\pm$\tiny{2.65} 
& 87.43$\pm$\tiny{2.01} 
& 76.23$\pm$\tiny{5.99} 
& 79.34$\pm$\tiny{0.94} 
& 64.92$\pm$\tiny{1.74} \\

AnomalyDAE 
& 65.75$\pm$\tiny{1.63} 
& 93.54$\pm$\tiny{1.88} 
& 91.99$\pm$\tiny{0.97} 
& 83.69$\pm$\tiny{0.29} 
& 62.45$\pm$\tiny{0.66} \\

GADNR 
& \cellcolor{red!10}67.01$\pm$\tiny{1.02} 
& 96.92$\pm$\tiny{0.43} 
& 93.46$\pm$\tiny{2.03} 
& 82.57$\pm$\tiny{0.84} 
& 70.11$\pm$\tiny{0.06} \\

ADAGAD 
& 66.45$\pm$\tiny{0.99} 
& \cellcolor{red!10}98.01$\pm$\tiny{0.01} 
& 96.05$\pm$\tiny{1.30} 
& \cellcolor{red!25}86.01$\pm$\tiny{0.45} 
& 71.06$\pm$\tiny{1.04} \\

\midrule

CurvGAD 
& \cellcolor{red!25}67.76$\pm$\tiny{1.54} 
& \cellcolor{red!25}99.04$\pm$\tiny{0.34}
& \cellcolor{red!25}97.50$\pm$\tiny{0.15} 
& 83.90$\pm$\tiny{1.25} 
& \cellcolor{red!25}73.89$\pm$\tiny{0.81} \\

CurvGAD + CCSGW
& \cellcolor{red!45}\textbf{70.55$\pm$\tiny{1.12}} 
& \cellcolor{red!45}\textbf{99.27$\pm$\tiny{0.09}} 
& \cellcolor{red!45}\textbf{97.73$\pm$\tiny{0.20}} 
& \cellcolor{red!45}\textbf{89.01$\pm$\tiny{0.46}} 
& \cellcolor{red!45}\textbf{76.40$\pm$\tiny{0.70}} \\

\textbf{$\Delta$ (CCSGW)} 
& 2.79
& 0.23 
& 0.23 
& 5.11 
& 2.51 \\

\bottomrule
\end{tabular}
\end{adjustbox}

\label{tab:curvgad_baselines}
\end{table*}

\textbf{Graph anomaly detection.}
Wegrate CCSGW into CurvGAD \citep{grover2025curvgad}, a state-of-the-art graph anomaly detector based on mixed-curvature representations, and evaluate on five GADBench datasets \citep{tang2023gadbench}: Reddit \citep{kumar2019predicting,liu2022bond}, Weibo \citep{zhao2020error}, Amazon \citep{mcauley2013amateurs}, Elliptic \citep{weber2019anti} and Questions \citep{platonov2023critical}.
CCSGW aligns the constant-curvature components of CurvGAD's product-manifold embedding.
At every epoch, CCSGW draws $256$ random nodes and, for each constant-curvature component, projects their embeddings along random geodesic directions (number of directions $L=32$) under the current curvature, which is kept fixed within this term.
The resulting slices are standardized and averaged over directions and component pairs.
This averaged discrepancy is added to the CurvGAD objective with the alignment weight $\lambda=10$, shared across all datasets.
Following the supervised protocol, we report test AUROC (mean $\pm$ 95\% CI) over ten splits.
We train the model with and without CCSGW in the training objective using all hyperparameters in \citet{grover2025curvgad}.
The results are summarized in \cref{tab:curvgad_baselines}, and the remaining baselines are adopted from \citet{grover2025curvgad}, with further details provided in \cref{app_exp_anomaly}.

\begin{table*}[t]
\caption{Node classification performance (weighted F1-score \%($\uparrow$) $\pm$ $95\%$ confidence interval).
\colorbox{red!45}{\textbf{First}}, \colorbox{red!25}{second}, and \colorbox{red!10}{third} best are highlighted.
The last row reports the gain of CCSGW.}
\centering
\normalsize
\begin{adjustbox}{max width=0.85\textwidth}
\begin{tabular}{@{}lccccc@{}}
\toprule
 & \multicolumn{3}{c}{\textbf{Homophilic}} & \multicolumn{2}{c}{\textbf{Heterophilic}} \\
\cmidrule(lr){2-4}\cmidrule(lr){5-6}
\textbf{Method} & {Cora} & {CiteSeer} & {PubMed} & {Chameleon} & {Actor} \\
\midrule

GCN
& 75.21\scriptsize{$\pm$0.28}
& 67.30\scriptsize{$\pm$1.05}
& 83.75\scriptsize{$\pm$0.07}
& 61.16\scriptsize{$\pm$0.23}
& 31.12\scriptsize{$\pm$0.96} \\

GAT
& 76.70\scriptsize{$\pm$0.13}
& 66.23\scriptsize{$\pm$0.85}
& 82.83\scriptsize{$\pm$0.22}
& 63.10\scriptsize{$\pm$0.77}
& 32.65\scriptsize{$\pm$0.23} \\

SAGE
& 71.88\scriptsize{$\pm$0.91}
& 70.01\scriptsize{$\pm$0.64}
& 81.09\scriptsize{$\pm$0.13}
& 59.99\scriptsize{$\pm$0.89}
& 36.73\scriptsize{$\pm$0.01} \\

\midrule

HGCN
& 78.50\scriptsize{$\pm$0.14}
& 69.55\scriptsize{$\pm$0.39}
& 83.72\scriptsize{$\pm$0.21}
& 60.18\scriptsize{$\pm$0.57}
& 35.89\scriptsize{$\pm$0.29} \\

HGAT
& 77.12\scriptsize{$\pm$0.01}
& 70.12\scriptsize{$\pm$0.92}
& 84.02\scriptsize{$\pm$0.19}
& 62.43\scriptsize{$\pm$0.59}
& 35.12\scriptsize{$\pm$0.27} \\

$\kappa$GCN
& 78.71\scriptsize{$\pm$1.37}
& 68.14\scriptsize{$\pm$0.34}
& 85.18\scriptsize{$\pm$0.52}
& 62.12\scriptsize{$\pm$0.49}
& 34.57\scriptsize{$\pm$0.26} \\

$\mathcal{Q}$GCN
& 79.64\scriptsize{$\pm$0.38}
& 71.15\scriptsize{$\pm$1.11}
& 84.76\scriptsize{$\pm$0.13}
& 61.83\scriptsize{$\pm$1.01}
& 32.24\scriptsize{$\pm$0.65} \\

SelfMGNN
& 80.19\scriptsize{$\pm$0.60}
& 70.91\scriptsize{$\pm$0.38}
& 82.81\scriptsize{$\pm$0.34}
& 64.97\scriptsize{$\pm$0.54}
& \colorbox{red!10}{38.99\scriptsize{$\pm$0.23}} \\

\midrule

ChebyNet
& 71.09\scriptsize{$\pm$0.91}
& 66.67\scriptsize{$\pm$0.38}
& 83.83\scriptsize{$\pm$0.42}
& 59.96\scriptsize{$\pm$1.51}
& 38.02\scriptsize{$\pm$0.01} \\

BernNet
& 73.34\scriptsize{$\pm$0.53}
& 62.12\scriptsize{$\pm$2.09}
& 82.15\scriptsize{$\pm$0.13}
& 62.03\scriptsize{$\pm$0.12}
& 33.55\scriptsize{$\pm$0.24} \\

GPRGNN
& 79.49\scriptsize{$\pm$0.31}
& 67.61\scriptsize{$\pm$0.38}
& 84.07\scriptsize{$\pm$0.09}
& 65.09\scriptsize{$\pm$0.43}
& 37.43\scriptsize{$\pm$1.09} \\

FiGURe
& 80.01\scriptsize{$\pm$0.09}
& 71.26\scriptsize{$\pm$0.41}
& 83.89\scriptsize{$\pm$0.11}
& \colorbox{red!10}{67.18\scriptsize{$\pm$0.02}}
& 38.31\scriptsize{$\pm$0.36} \\

\midrule

H2GCN
& 82.70\scriptsize{$\pm$0.90}
& 71.10\scriptsize{$\pm$0.80}
& 84.60\scriptsize{$\pm$0.50}
& 60.10\scriptsize{$\pm$2.20}
& 35.90\scriptsize{$\pm$1.20} \\

FAGCN
& 82.50\scriptsize{$\pm$0.50}
& 71.30\scriptsize{$\pm$0.60}
& 84.30\scriptsize{$\pm$0.40}
& 66.90\scriptsize{$\pm$1.80}
& 35.80\scriptsize{$\pm$1.10} \\

CurvDrop
& 82.50\scriptsize{$\pm$0.70}
& 72.80\scriptsize{$\pm$0.60}
& 85.20\scriptsize{$\pm$0.50}
& \colorbox{red!25}{67.30\scriptsize{$\pm$1.40}}
& \colorbox{red!25}{39.50\scriptsize{$\pm$1.00}} \\

ARGNN
& \colorbox{red!25}{86.65\scriptsize{$\pm$1.00}}
& \colorbox{red!45}{\textbf{76.29\scriptsize{$\pm$0.89}}}
& \colorbox{red!10}{90.15\scriptsize{$\pm$0.25}}
& 65.58\scriptsize{$\pm$2.05}
& 38.82\scriptsize{$\pm$0.92} \\

\midrule

CUSP
& \colorbox{red!10}{85.64\scriptsize{$\pm$1.48}}
& \colorbox{red!10}{74.10\scriptsize{$\pm$0.88}}
& \colorbox{red!25}{90.34\scriptsize{$\pm$0.21}}
& 66.48\scriptsize{$\pm$2.21}
& 38.90\scriptsize{$\pm$0.93} \\

CUSP\,+\,CCSGW
& \colorbox{red!45}{\textbf{86.66\scriptsize{$\pm$0.93}}}
& \colorbox{red!25}{74.24\scriptsize{$\pm$1.23}}
& \colorbox{red!45}{\textbf{90.39\scriptsize{$\pm$0.18}}}
& \colorbox{red!45}{\textbf{67.82\scriptsize{$\pm$1.92}}}
& \colorbox{red!45}{\textbf{39.83\scriptsize{$\pm$0.70}}} \\

\midrule

\textbf{$\Delta$ (CCSGW)}
& +1.02
& +0.14
& +0.06
& +1.34
& +0.93 \\

\bottomrule
\end{tabular}
\end{adjustbox}
\label{tab:results_node}
\end{table*}

\textbf{Graph node classification.}
We integrate CCSGW into CUSP~\citep{grover2025spectro}, a state-of-the-art spectral graph neural network on mixed-curvature product manifolds, and evaluate the resulting model on three homophilic datasets, Cora, CiteSeer, and PubMed~\citep{sen2008collective}, and two heterophilic datasets, Chameleon~\citep{pei2020geom} and Actor~\citep{tang2009social}.
As in CurvGAD, CCSGW aligns the constant-curvature components of CUSP's product-manifold node embeddings.
Following CUSP \citep{grover2025spectro}, we report the mean and 95\% confidence interval of test weighted F1 at the best validation epoch over ten random 60/20/20 splits.
We train the model with and without CCSGW in the training objective using all hyperparameters in \citet{grover2025spectro}.
The results are summarized in \cref{tab:results_node}, and the remaining baselines are adopted from \citet{grover2025spectro}, with further details provided in \cref{app_exp_node}.

\textbf{EEG-based multimodal learning.}
We integrate CCSGW into EEG-MoCE \citep{zhou2026eeg}, a state-of-the-art multimodal learning framework based on hyperbolic mixture-of-curvature experts, and evaluate the resulting model on the emotion recognition task \citep{lee2024eav} and two cognitive tasks \citep{shin2018simultaneous}.
For CCSGW alignment, we select one modality-specific hyperbolic manifold as the reference and align the embedding distribution of each remaining modality with that of the reference modality.
Following \citet{zhou2026eeg}, we perform subject-wise 10-fold cross-validation on EAV ($n=42$ participants) and Cognitive ($n=26$ participants), and report test balanced accuracy averaged over folds.
We train the model with and without CCSGW in the training objective using all hyperparameters in \citet{zhou2026eeg}.
The results are summarized in \cref{tab:multimodal_results}, and the remaining baselines are adopted from \citet{zhou2026eeg}, with details in \cref{app_exp_eeg}.
\begin{table}[t]
\centering
\caption{
Performance comparison on the emotion recognition and the
Cognitive tasks. We report balanced accuracy.
EAV uses cross-subject 10-fold cross-validation, while the
Cognitive tasks use cross-subject 10-fold cross-validation.
\colorbox{red!45}{\textbf{First}}, \colorbox{red!25}{second}, and \colorbox{red!10}{third} best are highlighted.
}
\label{tab:multimodal_results}
\vspace{4pt}

\normalsize
\setlength{\tabcolsep}{5pt}
\setlength{\fboxsep}{1.5pt}

\begin{adjustbox}{max width=\textwidth}
\begin{tabular}{lc@{\hspace{1.8em}}lcc}
\toprule
\multicolumn{2}{c}{\textbf{Emotion}} &
\multicolumn{3}{c}{\textbf{Cognitive}} \\
\cmidrule(lr){1-2}\cmidrule(lr){3-5}
Method & Emotion (\%) &
Method & N-back & Word Generation \\
\midrule

AGF-IB
& 54.59\scriptsize{$\pm$12.45}
& CTMWA
& 49.86\scriptsize{$\pm$11.72}
& 59.87\scriptsize{$\pm$8.15} \\

CMERC
& 58.57\scriptsize{$\pm$11.95}
& EFDFNet
& 49.15\scriptsize{$\pm$12.85}
& 58.85\scriptsize{$\pm$7.63} \\

CTMWA
& 52.86\scriptsize{$\pm$5.63}
& EF-Net
& \colorbox{red!10}{54.41\scriptsize{$\pm$14.03}}
& 64.53\scriptsize{$\pm$11.08} \\

GA2MIF
& 61.10\scriptsize{$\pm$12.51}
& LMF
& 52.71\scriptsize{$\pm$11.77}
& 59.74\scriptsize{$\pm$6.18} \\

HEEGNet
& \colorbox{red!10}{61.74\scriptsize{$\pm$8.33}}
& MMML
& 45.32\scriptsize{$\pm$9.37}
& 56.35\scriptsize{$\pm$8.22} \\

Hyper-MML
& 60.76\scriptsize{$\pm$12.15}
& STA-Net
& 46.63\scriptsize{$\pm$11.67}
& 62.50\scriptsize{$\pm$8.72} \\

LMF
& 57.29\scriptsize{$\pm$12.37}
& ST2A
& 51.28\scriptsize{$\pm$11.74}
& \colorbox{red!10}{67.27\scriptsize{$\pm$10.45}} \\

MMML
& 59.60\scriptsize{$\pm$11.83}
& TSMMF
& 53.84\scriptsize{$\pm$13.62}
& 63.82\scriptsize{$\pm$9.47} \\

\midrule

EEGMOCE
& \colorbox{red!25}{76.45\scriptsize{$\pm$9.82}}
& EEGMOCE
& \colorbox{red!25}{60.26\scriptsize{$\pm$17.29}}
& \colorbox{red!25}{71.15\scriptsize{$\pm$11.90}} \\

EEGMOCE + CCSGW
& \colorbox{red!45}{\textbf{77.14\scriptsize{$\pm$9.72}}}
& EEGMOCE + CCSGW
& \colorbox{red!45}{\textbf{61.68\scriptsize{$\pm$17.24}}}
& \colorbox{red!45}{\textbf{72.76\scriptsize{$\pm$13.48}}} \\

\textbf{$\Delta$ (CCSGW)}
& +0.69
& \textbf{$\Delta$ (CCSGW)}
& +1.42
& +1.61 \\

\bottomrule
\end{tabular}
\end{adjustbox}
\end{table}

\textbf{Results.}
CCSGW consistently improves the mean performance of all three backbone models across the evaluated datasets and tasks.
For graph anomaly detection (\cref{tab:curvgad_baselines}), it improves CurvGAD by up to $5.11$ percentage points in AUROC on Elliptic.
For graph learning (\cref{tab:results_node}), CUSP improves weighted F1 by $1.02$ and $1.34$ percentage points on Cora and Chameleon, respectively.
For EEG-based multimodal learning (\cref{tab:multimodal_results}), CCSGW increases balanced accuracy by $0.69$, $1.42$, and $1.61$ percentage points on emotion recognition, N-back, and word generation, respectively.
These consistent improvements support the effectiveness of CCSGW as a distribution-level alignment regularizer for heterogeneous cross-curvature representations.
Moreover, the gains obtained while retaining the backbone architectures and manifold configurations suggest that explicitly aligning component distributions complements the representational flexibility of mixed-curvature models.
Specifically, CCSGW encourages consistency in the relational structure of embedding distributions while allowing each component to retain its own curvature. 
However, the varying improvements across datasets also suggest that the benefit of such alignment may depend on the quality of the sampled geodesic projections, which may not fully preserve the geometry of the original high-dimensional manifolds.
In the appendix, we further show that these gains are robust to the alignment weight, the number of projection directions and the number of sampled nodes (\cref{app_sec: sensitivity}).
Additionally, using CUSP dataset as an example, we show that CCSGW outperforms SGW and matches GW in accuracy (\cref{app_sec: align_compare}), and its cost is the same as that of SGW (\cref{app_sec: cost}).

\section{Discussion}

In this work, we proposed CCSGW, a geometry-aware divergence for aligning distributions across heterogeneous constant-curvature spaces. 
We first introduce geodesic-based one-dimensional projections for spherical spaces and extending sliced GW to mixed-curvature settings. 
Our theoretical analysis shows that minimizing CCSGW controls the intrinsic GW discrepancy up to projection error, providing a principled connection between the sliced objective and distribution-level geometric consistency. 
Empirically, incorporating CCSGW into mixed-curvature models consistently improves performance across graph anomaly detection, node classification, and multimodal learning, suggesting that cross-curvature consistency can serve as a useful inductive bias across different architectures and tasks. 
The limitation of CCSGW is that its effectiveness depends on the quality of the sampled geodesic projections, which may not fully preserve the geometry of the original high-dimensional manifolds.
Future work could therefore explore subspace alignment for more robust cross-curvature matching by preserving richer geometric information beyond one-dimensional projections.
Extending CCSGW beyond constant-curvature product spaces to more general heterogeneous manifolds is another promising direction, potentially providing a broader framework for geometry-aware representation alignment.

\clearpage
\clearpage
\bibliography{example_paper}
\bibliographystyle{iclr2027_conference}

\clearpage
\appendix

\section{Sphere geodesic projections}
\label{appendix Sphere geodesic projections}

Geodesic Projection onto a Fixed Geodesic on the Sphere (Radius Model)

1. Geometric Setting

Let $K > 0$ and consider the sphere (radius model)
$$S_K^n = \{ x \in \mathbb{R}^{n+1} : \|x\|^2 = 1 / K \}.$$

Fix:
- a reference point $x_0 \in S_K^n$,
- a unit tangent vector $v \in T_{x_0} S_K^n$ with $\|v\| = 1$.

The geodesic starting from $x_0$ in direction $v$ is given by the exponential map
$$\gamma(t) = \mathrm{Exp}_{x_0}(t v) = \cos(\sqrt{K}\, t)\, x_0 + \frac{\sin(\sqrt{K}\, t)}{\sqrt{K}}\, v,\qquad t \in \mathbb{R}.$$

For any $x \in S_K^n$, define the geodesic projection parameter
$$t^\star(x) \in \operatorname*{arg\,min}_{t \in \mathbb{R}} d(x, \gamma(t)).$$
where the geodesic distance on $S_K^n$ is
$$d(x,y) = \frac{1}{\sqrt{K}}\,\arccos\!\big( K \langle x, y \rangle \big).$$

2. Theorem (Closed-form Geodesic Projection Parameter)

Let $x \in S_K^n$. The optimal projection parameter $t^\star(x)$ satisfies the following case-wise characterization.

Case 1: Generic case $\langle x, x_0 \rangle \neq 0$

If $\langle x, x_0 \rangle \neq 0$,
then any optimal parameter $t^\star(x)$ satisfies
$$\tan(\sqrt{K}\, t^\star(x)) = \frac{\langle x, v \rangle}{\sqrt{K}\, \langle x, x_0 \rangle}.$$

A consistent choice of branch is
$$t^\star(x) = \frac{1}{\sqrt{K}}\,\mathrm{atan2}\big(\langle x, v \rangle,\ \sqrt{K}\,\langle x, x_0 \rangle\big).$$

---

Case 2: Equatorial case $\langle x, x_0 \rangle = 0$ and $\langle x, v \rangle \neq 0$

If $\langle x, x_0 \rangle = 0$ and $\langle x, v \rangle \neq 0$,
then an optimal projection parameter is
$$t^\star(x) = \frac{1}{\sqrt{K}}\,\arcsin\big(\mathrm{sign}(\langle x, v \rangle)\big).$$
$$t^\star(x) = \frac{1}{\sqrt{K}}\begin{cases}\frac{\pi}{2}, & \langle x, v \rangle > 0, \\ -\frac{\pi}{2}, & \langle x, v \rangle < 0.\end{cases}$$

Case 3: Degenerate case $\langle x, x_0 \rangle = 0$ and $\langle x, v \rangle = 0$

If $\langle x, x_0 \rangle = 0$ and $\langle x, v \rangle = 0$,
then
$$\langle x, \gamma(t) \rangle = 0,\qquad \text{for all } t \in \mathbb{R}.$$
and the projection is not unique: every $t \in \mathbb{R}$ is optimal.

3.Proof

Step 1: Rewrite the objective using the explicit distance

By definition,
$$t^\star(x) \in \operatorname*{arg\,min}_{t \in \mathbb{R}} d(x,\gamma(t)).$$

Using $$d(x,y) = \frac{1}{\sqrt{K}}\,\arccos\!\big(K\langle x,y\rangle\big),$$ we have
$$d(x,\gamma(t)) = \frac{1}{\sqrt{K}}\,\arccos\!\big(K\langle x,\gamma(t)\rangle\big).$$

Since the factor $1/\sqrt{K}$ is positive and does not depend on $t$, we may drop it inside $\operatorname*{arg\,min}$:
$$t^\star(x) \in \operatorname*{arg\,min}_{t \in \mathbb{R}} \arccos\!\big(K\langle x,\gamma(t)\rangle\big).$$

Step 2: Use monotonicity of $\arccos$

On $[-1,1]$, $\arccos(\cdot)$ is strictly decreasing. Therefore, minimizing $\arccos(\cdot)$ is equivalent to maximizing its argument. Hence,
$$t^\star(x) \in \operatorname*{arg\,max}_{t \in \mathbb{R}} K\langle x,\gamma(t)\rangle.$$

Since $K>0$ is constant, maximizing $K\langle x,\gamma(t)\rangle$ is equivalent to maximizing $\langle x,\gamma(t)\rangle$:
$$t^\star(x) \in \operatorname*{arg\,max}_{t \in \mathbb{R}} \langle x,\gamma(t)\rangle.$$

Step 3: Expand the objective $\langle x,\gamma(t)\rangle$

Recall
$$\gamma(t) = \cos(\sqrt{K}\,t)\,x_0 + \frac{\sin(\sqrt{K}\,t)}{\sqrt{K}}\,v.$$

Then, by bilinearity of the Euclidean inner product,
$$\langle x,\gamma(t)\rangle = \left\langle x,\cos(\sqrt{K}\,t)\,x_0 + \frac{\sin(\sqrt{K}\,t)}{\sqrt{K}}\,v \right\rangle.$$

Pulling out scalar coefficients,
$$\langle x,\gamma(t)\rangle = \cos(\sqrt{K}\,t)\,\langle x,x_0\rangle + \frac{\sin(\sqrt{K}\,t)}{\sqrt{K}}\,\langle x,v\rangle.$$

Step 4: Differentiate with respect to $t$

Differentiate term by term. Using
$$\frac{d}{dt}\cos(\sqrt{K}\,t) = -\sqrt{K}\,\sin(\sqrt{K}\,t),$$
and
$$\frac{d}{dt}\sin(\sqrt{K}\,t) = \sqrt{K}\,\cos(\sqrt{K}\,t),$$
we obtain
$$\frac{d}{dt}\langle x,\gamma(t)\rangle = -\sqrt{K}\,\sin(\sqrt{K}\,t)\,\langle x,x_0\rangle + \cos(\sqrt{K}\,t)\,\langle x,v\rangle.$$

At any interior maximizer $t^\star(x)$ where the derivative exists, we must have
$$-\sqrt{K}\,\sin(\sqrt{K}\,t^\star(x))\,\langle x,x_0\rangle + \cos(\sqrt{K}\,t^\star(x))\,\langle x,v\rangle = 0.$$

Equivalently,
$$\cos(\sqrt{K}\,t^\star(x))\,\langle x,v\rangle = \sqrt{K}\,\sin(\sqrt{K}\,t^\star(x))\,\langle x,x_0\rangle.$$

---

Step 5: Solve the optimality condition (case-wise)

Case 1: $\langle x,x_0\rangle \neq 0$

Assume $\langle x,x_0\rangle \neq 0$. If $\cos(\sqrt{K}\,t^\star(x)) \neq 0$, divide both sides by $\sqrt{K}\,\langle x,x_0\rangle\,\cos(\sqrt{K}\,t^\star(x))$ to get
$$\tan(\sqrt{K}\,t^\star(x)) = \frac{\langle x,v\rangle}{\sqrt{K}\,\langle x,x_0\rangle}.$$

A consistent branch choice that also handles the correct quadrant is
$$t^\star(x) = \frac{1}{\sqrt{K}}\,\mathrm{atan2}\big(\langle x,v\rangle,\ \sqrt{K}\,\langle x,x_0\rangle\big).$$

Case 2: $\langle x,x_0\rangle = 0$ and $\langle x,v\rangle \neq 0$

Assume $\langle x,x_0\rangle = 0$ and $\langle x,v\rangle \neq 0$. The optimality condition becomes
$$\cos(\sqrt{K}\,t^\star(x))\,\langle x,v\rangle = 0,$$
hence
$$\cos(\sqrt{K}\,t^\star(x)) = 0.$$

Therefore $\sqrt{K}\,t^\star(x) = \frac{\pi}{2} + m\pi$ for some $m \in \mathbb{Z}$. To choose the maximizer, note that when $\langle x,x_0\rangle=0$ the objective reduces to
$$\langle x,\gamma(t)\rangle = \frac{\sin(\sqrt{K}\,t)}{\sqrt{K}}\,\langle x,v\rangle.$$

This is maximized by choosing $\sin(\sqrt{K}\,t^\star(x)) = \mathrm{sign}(\langle x,v\rangle)$, hence one convenient representative is
$$t^\star(x) = \frac{1}{\sqrt{K}}\,\arcsin\big(\mathrm{sign}(\langle x,v\rangle)\big).$$

Case 3: $\langle x,x_0\rangle = 0$ and $\langle x,v\rangle = 0$

Assume $\langle x,x_0\rangle = 0$ and $\langle x,v\rangle = 0$. Then for all $t$,
$$\langle x,\gamma(t)\rangle = 0,$$
so every $t \in \mathbb{R}$ is optimal (the projection is not unique).

This completes the proof.

4. Remarks
The generic case holds for almost every $x \in S_K^n$ except a measure-zero set.

\clearpage

\section{Theoretical analysis proofs}
\label{app_proofs}
 
\subsection{Setting}
\label{app:notation}
 
\begin{assumption}\label{ass:geom}
Each distribution is supported in the closed ball $\bar B(o,R)$ around the
origin $o$ of its component, with $R<\infty$. For a spherical component with
curvature $K>0$, additionally
\begin{equation}\label{eq:cap}
    R<\frac{\pi}{2\sqrt K}.
\end{equation}
\end{assumption}
 
We write $d_X,d_Y$ for the geodesic distances of the ambient
constant-curvature spaces. For
$\pi\in\Pi(\mu,\nu)$, the distortion is defined as:
\begin{equation}\label{eq:dis}
    \operatorname{dis}(\pi):=\big\|d_X(x,x')-d_Y(y,y')\big\|_{L_2(\pi\otimes\pi)},
\end{equation}
so that $\operatorname{GW}_2(\mathcal{X},\mathcal{Y})
=\inf_{\pi\in\Pi(\mu,\nu)}\operatorname{dis}(\pi)$, attained by $\pi^\star$
on compact supports \citep{memoli2011gromov}.
 
For a unit tangent direction $v$ at $o_X$, let $\gamma_v$ be the unit-speed
geodesic with $\gamma_v(0)=o_X$ and $\dot\gamma_v(0)=v$; $P_v^X(x)$ is the
point of $\gamma_v$ closest to $x$, and $p_v^X(x)$ is defined by
$P_v^X(x)=\gamma_v(p_v^X(x))$. $P_w^Y,p_w^Y$ follows similar definition analogously. With
$\mu_v:=(p_v^X)_\#\mu$ and $\nu_w:=(p_w^Y)_\#\nu$, let
$Q(v,w):=\operatorname{GW}_2^2(\mu_v,\nu_w)$, so that
$S=\mathbb{E}_{v,w}Q(v,w)$ with $v\sim\lambda_X$, $w\sim\lambda_Y$
independent. Finally, define:
\begin{align}
    d_v^X(x,x')&:=\big|p_v^X(x)-p_v^X(x')\big|,
    &h_v^X(x)&:=d_X\big(x,P_v^X(x)\big),\label{eq:H}\\
    \delta_X(v)&:=\big\|d_X-d_v^X\big\|_{L_2(\mu\otimes\mu)},
    &\Delta_X&:=\big(\mathbb{E}_v\,\delta_X(v)^2\big)^{1/2},\label{eq:delta}
\end{align}
with $H_X=(\mathbb{E}_v\mathbb{E}_{x\sim\mu}h_v^X(x)^2)^{1/2}$ as in
\cref{sec:theory}, and $d_w^Y,h_w^Y,\delta_Y,\Delta_Y$ analogously.
 
\subsection{Proof of \texorpdfstring{\cref{thm:ccsgw_gw}}{Theorem 1}}
\label{app:proof_gw}
 
\begin{lemma}\label{lem:coord}
Under \cref{ass:geom}, for every direction $v$ and all
$x,x'\in\bar B(o,R)$, the projection $P_v(x)$ is unique,
$|p_v(x)|\le d_K(x,o)\le R$, and:
\begin{equation}\label{eq:isometric}
    \big|p_v(x)-p_v(x')\big|=d_K\big(P_v(x),P_v(x')\big).
\end{equation}
\end{lemma}
 
\begin{proof}
\emph{Case $K\le0$.} Euclidean and hyperbolic spaces are Cartan--Hadamard
manifolds, for which the projection onto a geodesic is unique, the signed
coordinate is an isometry onto $\mathbb{R}$, which is \cref{eq:isometric},
and $p_v$ is $1$-Lipschitz \citep[Sec.~3.2]{bonet2025sliced}. Since $p_v(o)=0$, this gives
$|p_v(x)|\le d_K(x,o)\le R$.
 
\emph{Case $K>0$.} Let \(k=\sqrt K\) and \(E=\operatorname{span}\{o,v\}\). By the great-circle projection formula of Bonet et al. (2023, Lemma 1), rescaled to radius \(1/k\),
\[
P_v(x)=\frac{\operatorname{proj}_E x}
{k\|\operatorname{proj}_E x\|}.
\]
Since \(d_K(o,x)\le R<\pi/(2k)\), we have \(\langle o,x\rangle>0\), so this projection is well defined and unique. Moreover, because \(\|\operatorname{proj}_E x\|\le\|x\|=1/k\),
\[
\cos\!\bigl(kd_K(o,P_v(x))\bigr)
=\frac{k\langle o,x\rangle}{\|\operatorname{proj}_E x\|}
\ge k^2\langle o,x\rangle
=\cos\!\bigl(kd_K(o,x)\bigr).
\]
Thus \(|p_v(x)|=d_K(o,P_v(x))\le d_K(o,x)\le R\). Consequently, any two projected coordinates differ by at most \(2R<\pi/k\), and hence their absolute difference equals the shortest great-circle distance between the projected points, proving Eq. (16).
\end{proof}
 
\begin{lemma}\label{lem:proj}
Under \cref{ass:geom}, $\Delta_X\le 2H_X$ and $\Delta_Y\le 2H_Y$.
\end{lemma}
 
\begin{proof}
Fix $v,x,x'$ and write $P=P_v^X(x)$, $P'=P_v^X(x')$. By \cref{lem:coord},
$d_v^X(x,x')=d_X(P,P')$. The triangle inequality along
$x\to P\to P'\to x'$ gives
\[
    d_X(x,x')\le d_X(x,P)+d_X(P,P')+d_X(P',x')
    =d_v^X(x,x')+h_v^X(x)+h_v^X(x'),
\]
and along $P\to x\to x'\to P'$ gives
\[
    d_v^X(x,x')=d_X(P,P')\le d_X(P,x)+d_X(x,x')+d_X(x',P')
    =d_X(x,x')+h_v^X(x)+h_v^X(x').
\]
Together,
\begin{equation}\label{eq:pointwise_proj}
    \big|d_X(x,x')-d_v^X(x,x')\big|\le h_v^X(x)+h_v^X(x').
\end{equation}
Taking the $L_2(\mu\otimes\mu)$ norm and applying Minkowski's inequality,
since $x$ and $x'$ each have law $\mu$,
$\delta_X(v)\le2\|h_v^X\|_{L_2(\mu)}$. Squaring and averaging over $v$ gives
$\Delta_X^2\le4H_X^2$. The same argument applies to $\mathcal{Y}$.
\end{proof}
 
\begin{lemma}\label{lem:slice}
Under \cref{ass:geom}, for every $v,w$,
\begin{equation}\label{eq:slice}
\operatorname{GW}_2(\mathcal{X},\mathcal{Y})
\le \delta_X(v)+\sqrt{Q(v,w)}+\delta_Y(w).
\end{equation}
\end{lemma}

\begin{proof}
Define the projected metric-measure spaces
\[
\mathcal{X}_v
:= \bigl(\operatorname{supp}\mu_v,|\cdot-\cdot|,\mu_v\bigr),
\qquad
\mathcal{Y}_w
:= \bigl(\operatorname{supp}\nu_w,|\cdot-\cdot|,\nu_w\bigr).
\]
Since $\mu_v=(p_v^X)_\#\mu$, the measure
$\eta_X=(\operatorname{id},p_v^X)_\#\mu$ is a coupling of
$\mu$ and $\mu_v$. Evaluating its distortion yields
\[
\operatorname{GW}_2(\mathcal{X},\mathcal{X}_v)
\le
\left\|d_X-d_v^X\right\|_{L_2(\mu\otimes\mu)}
=\delta_X(v).
\]
Similarly, the coupling
$\eta_Y=(\operatorname{id},p_w^Y)_\#\nu$ gives
\[
\operatorname{GW}_2(\mathcal{Y}_w,\mathcal{Y})
\le \delta_Y(w).
\]
By the triangle inequality for GW
(M\'emoli, 2011, Theorem~5.1(a)),
\[
\begin{aligned}
\operatorname{GW}_2(\mathcal{X},\mathcal{Y})
&\le
\operatorname{GW}_2(\mathcal{X},\mathcal{X}_v)
+\operatorname{GW}_2(\mathcal{X}_v,\mathcal{Y}_w)
+\operatorname{GW}_2(\mathcal{Y}_w,\mathcal{Y})\\
&\le \delta_X(v)+\sqrt{Q(v,w)}+\delta_Y(w),
\end{aligned}
\]
where
$\operatorname{GW}_2(\mathcal{X}_v,\mathcal{Y}_w)
=\sqrt{Q(v,w)}$ by definition.
\end{proof}

\begin{proof}[Proof of \cref{thm:ccsgw_gw}]
The left side of \cref{eq:slice} does not depend on $(v,w)$.
Averaging the right side over $v\sim\lambda_X$,
$w\sim\lambda_Y$ and using
$\mathbb{E}Z\le(\mathbb{E}Z^2)^{1/2}$ for nonnegative $Z$,
\[
\begin{aligned}
\operatorname{GW}_2(\mathcal{X},\mathcal{Y})
&\le
\mathbb{E}_v\delta_X(v)
+\mathbb{E}_{v,w}\sqrt{Q(v,w)}
+\mathbb{E}_w\delta_Y(w)\\
&\le \Delta_X+\sqrt{S}+\Delta_Y\\
&\le 2H_X+\sqrt{S}+2H_Y,
\end{aligned}
\]
where the last step follows from \cref{lem:proj}.
\end{proof}
 
\subsection{Proof of \texorpdfstring{\cref{thm:average_isometry}}{Theorem 2}}
\label{app:proof_iso}
 
\begin{lemma}\label{lem:L2}
For every $\pi\in\Pi(\mu,\nu)$ and every measurable $g:X\to Y$,
\begin{equation}\label{eq:L2}
    \big\|d_Y\big(g(x),g(x')\big)-d_X(x,x')\big\|_{L_2(\pi\otimes\pi)}
    \le\operatorname{dis}(\pi)+2\rho_\pi(g).
\end{equation}
\end{lemma}
 
\begin{proof}
Fix $(x,y),(x',y')$. The triangle inequality in $Y$ along
$g(x)\to y\to y'\to g(x')$ gives
\[
    d_Y(g(x),g(x'))-d_X(x,x')
    \le\big[d_Y(y,y')-d_X(x,x')\big]+d_Y(g(x),y)+d_Y(g(x'),y'),
\]
and along $y\to g(x)\to g(x')\to y'$ gives
\[
    d_X(x,x')-d_Y(g(x),g(x'))
    \le\big[d_X(x,x')-d_Y(y,y')\big]+d_Y(g(x),y)+d_Y(g(x'),y').
\]
Together,
\begin{equation}\label{eq:pointwise_L2}
    \big|d_Y(g(x),g(x'))-d_X(x,x')\big|
    \le\big|d_Y(y,y')-d_X(x,x')\big|+d_Y(g(x),y)+d_Y(g(x'),y').
\end{equation}
Taking the $L_2(\pi\otimes\pi)$ norm and applying Minkowski's inequality,
the first term has norm $\operatorname{dis}(\pi)$, and each of the other
two has norm $\rho_\pi(g)$ since $(x,y)$ and $(x',y')$ each have law $\pi$.
\end{proof}
 
\begin{proof}[Proof of \cref{thm:average_isometry}]
For $\pi^\star\otimes\pi^\star$-almost every pair with $x\ne x'$, we have
$x,x'\in S_X$ and hence $d_X(x,x')\ge r$, so
\[
    \left|\frac{d_Y(g(x),g(x'))}{d_X(x,x')}-1\right|
    =\frac{\big|d_Y(g(x),g(x'))-d_X(x,x')\big|}{d_X(x,x')}
    \le\frac{\big|d_Y(g(x),g(x'))-d_X(x,x')\big|}{r}.
\]
Squaring and integrating over $\{x\ne x'\}$, the left side of the first
claim is at most
\[
    r^{-1}\big\|d_Y(g(x),g(x'))-d_X(x,x')\big\|_{L_2(\pi^\star\otimes\pi^\star)}.
\]
By \cref{lem:L2} with $\pi=\pi^\star$, $\operatorname{dis}(\pi^\star)
=\operatorname{GW}_2(\mathcal{X},\mathcal{Y})$, and \cref{thm:ccsgw_gw},
this norm is at most $B$, proving the first claim. For the second, let $Z$
be the dilation deviation on $\{x\ne x'\}$ and $t:=B/(r\sqrt\epsilon)$; by
Chebyshev's inequality,
\[
    (\pi^\star\otimes\pi^\star)\big(\{x\ne x',\,|Z|>t\}\big)
    \le(B/r)^2/t^2=\epsilon.
\]
\end{proof}

\clearpage
\section{Experiment details}

\subsection{Graph anomaly detection}
\label{app_exp_anomaly}
We evaluate on five GADBench datasets~\citep{tang2023gadbench}: Reddit~\citep{kumar2019predicting,liu2022bond}, Weibo~\citep{zhao2020error}, Amazon~\citep{mcauley2013amateurs}, Elliptic~\citep{weber2019anti} and Questions~\citep{platonov2023critical}.
Following the fully-supervised protocol, we report the mean and 95\% confidence interval of test AUROC at the best validation epoch over ten splits.
CurvGAD~\citep{grover2025curvgad} retains its original hyperparameters and manifold signatures, and all of its results are reproduced under this protocol.
At every epoch, CCSGW draws $256$ random nodes and, for each constant-curvature component, projects their embeddings along random geodesic directions (number of directions $L=32$) under the current curvature, which is kept fixed within this term.
The resulting slices are standardized, and their closed-form one-dimensional Gromov--Wasserstein discrepancy is averaged over directions and component pairs.
This averaged discrepancy is added to the CurvGAD objective with the alignment weight $\lambda=10$, shared across all datasets.

The baselines we consider for graph anomaly detection tasks are:  GCN \citep{kipf2016semi}, GAT \citep{velivckovic2017graph}, SAGE \citep{hamilton2017inductive}, BernNet \citep{he2021bernnet},  ChebyNet \citep{defferrard2016convolutional}, HGCN \citep{chami2019hyperbolic}, HGAT \citep{zhang2021hyperbolic}, KGCN \citep{bachmann2020constant}, QGCN \citep{xiong2022pseudo}, DCI \citep{wang2021decoupling}, PCGNN \citep{liu2021pick}, BWGNN \citep{tang2022rethinking}, DOMINANT \citep{ding2019deep}, AnomalyDAE \citep{fan2020anomalydae}, GADNR \citep{roy2024gad}, ADAGAD \citep{he2024ada}.

\subsection{Graph node classification}
\label{app_exp_node}
We evaluate on three homophilic datasets, Cora, CiteSeer and PubMed~\citep{sen2008collective}, and two heterophilic datasets, Chameleon~\citep{pei2020geom} and Actor~\citep{tang2009social}.
Following CUSP~\citep{grover2025spectro}, we report the mean and 95\% confidence interval of test weighted F1 at the best validation epoch over ten random 60/20/20 splits.
CUSP strictly follows the configuration of the original paper: each dataset uses the reported manifold signature, i.e., the type and dimension of every constant-curvature component, namely $\mathbb{H}^{16}\times\mathbb{S}^{16}\times\mathbb{E}^{16}$ for Cora and CiteSeer, $\mathbb{H}^{16}\times\mathbb{H}^{16}\times\mathbb{E}^{16}$ for PubMed and Actor, and $\mathbb{H}^{16}\times\mathbb{S}^{16}\times\mathbb{S}^{16}$ for Chameleon, with learnable curvatures initialized as in the original paper.
All remaining hyperparameters, including the number of training epochs ($300$), are taken from the original search space, and all CUSP results are reproduced under this protocol with early stopping on validation F1.
At every epoch, CCSGW draws $256$ random nodes for homophilic and $512$ for heterophilic datasets and, for each constant-curvature component, projects their embeddings along random geodesic directions (number of directions $L=32$) under the current curvature, which also receives gradients from this term.
The resulting slices are standardized, and their closed-form one-dimensional Gromov--Wasserstein discrepancy is averaged over directions and component pairs.
This averaged discrepancy is added to the CUSP objective with the alignment weight $\lambda=1$, shared across all datasets and linearly warmed up over the first 50 epochs.

The baselines we consider for graph node classification tasks are: GCN \citep{kipf2016semi}, GAT \citep{velivckovic2017graph}, SAGE \citep{hamilton2017inductive}, HGCN \citep{chami2019hyperbolic}, HGAT \citep{zhang2021hyperbolic}, KGCN \citep{bachmann2020constant}, QGCN \citep{xiong2022pseudo}, SelfMGNN \citep{sun2022self}, ChebyNet \citep{defferrard2016convolutional}, BernNet \citep{he2021bernnet}, GPRGNN \citep{chien2020adaptive}, FiGURe \citep{ekbote2023figure}, H2GCN \citep{zhu2020beyond}, FAGCN \citep{bo2021beyond}, CurvDrop \citep{liu2023curvdrop}, ARGNN \citep{wang2026adaptive}.

\subsection{EEG-based multimodal learning}
\label{app_exp_eeg}
EEG-based multimodal learning combines brain signals with complementary modalities to support tasks such as emotion recognition, sleep staging, and cognitive assessment.
We further integrate CCSGW into EEG-MoCE \citep{zhou2026eeg}, a state-of-the-art multimodal learning framework based on hyperbolic mixture-of-curvature experts, and evaluate the resulting model on the emotion recognition tasks and two cognitive tasks.
Following \citep{zhou2026eeg}, we perform subject-wise 10-fold cross-validation on EAV ($n=42$ participants) and Cognitive ($n=26$ participants), and report test balanced accuracy (BAcc) averaged over folds.
Both EEG-MoCE and its CCSGW-augmented variant follow the original preprocessing pipeline and model configuration, including the modality-specific encoders, hyperbolic experts with learnable curvatures, and curvature-aware fusion module.
All optimization settings follow the original paper.
For CCSGW alignment, we select one modality-specific hyperbolic manifold as the reference and align the embedding distribution of each remaining modality with that of the reference modality.
The modality-specific embeddings are projected along random geodesic directions under their respective curvatures, and the resulting one-dimensional slices are standardized.
The closed-form one-dimensional Gromov--Wasserstein discrepancies are then averaged over projection directions and pairs formed by the reference modality and each remaining modality.
This averaged discrepancy is added to the original EEG-MoCE objective as a weighted alignment regularizer.

The baselines we consider for emotion recognition tasks are: AGF-IB~\cite{shou2024adversarial}, CMERC~\cite{tu2024calibrate},  CTMWA~\cite{zhang2024ctmwa}, GA2MIF~\cite{li2024ga2mif}, HEEGNet~\cite{li2026heegnet}, Hyper-MML~\cite{kang2025hypergraph},LMF~\cite{liu2018efficient}, MMML~\cite{wu2024mmml}; for the cognitive tasks are: CTMWA~\cite{zhang2024ctmwa}, EFDFNet~\cite{xu2025efdfnet}, EF-Net~\cite{arif2024efnet}, LMF~\cite{liu2018efficient}, MMML~\cite{wu2024mmml}, STA-Net~\cite{liu2025stanet}, ST2A~\cite{shi2025st2a}, TSMMF~\cite{si2025tsmmf} .

\section{Additional experiments}
\label{app_sec: additional}

\subsection{Hyperparameter sensitivity}
\label{app_sec: sensitivity}
We study the sensitivity of CCSGW to its hyperparameters on Cora with CUSP and on Reddit with CurvGAD.
In both backbones, node embeddings lie in a mixed-curvature space with $m$ constant-curvature components, and CCSGW aligns every pair of components.
Let $\mu_j$ denote the distribution of node embeddings in the $j$-th component, whose curvature is $K_j$.
The training objective is
\begin{equation}
\label{app_eq:objective}
\mathcal{L}
= \mathcal{L}_{\mathrm{task}}
+ \lambda \binom{m}{2}^{-1} \sum_{1\le j<k\le m}
\frac{1}{L}\sum_{\ell=1}^{L}
\mathrm{GW}_2^2\!\left(d_{\mathbb{R}}^2,\,
(p^{K_j}_{v_j^{\ell}})_\#\hat\mu_j,\,
(p^{K_k}_{v_k^{\ell}})_\#\hat\mu_k\right),
\end{equation}
where $\hat\mu_j$ is the empirical measure of $\mu_j$ on $n$ nodes sampled uniformly at random (the same nodes for all components), and $v_j^{1},\dots,v_j^{L}$ are i.i.d.\ uniform unit tangent directions at the origin of the $j$-th component.
Nodes and directions are redrawn at every epoch, and each projected slice is standardized as described in \cref{app_exp_anomaly}.
As in SGW (\cref{eq:sgw}), the integral over directions is replaced by an average over $L$ random directions.
CCSGW therefore has three hyperparameters: the alignment weight $\lambda$, the number of geodesic directions $L$, which controls the accuracy of this approximation, and the number of sampled nodes $n$, which sets the support size of the empirical measures.
Each component pair costs $O(L\,n\log n)$.

We vary one hyperparameter at a time around a reference configuration $(\lambda, L, n)$, set to $(1, 32, 512)$ on Cora and $(10, 32, 256)$ on Reddit, over the values on the horizontal axes of \cref{app_fig:sensitivity}.
All other settings, including the ten splits, follow \cref{app_exp_node} for Cora and \cref{app_exp_anomaly} for Reddit.
For each configuration, we report the gain $\Delta$ of the backbone trained with CCSGW over the same backbone trained without it ($\lambda=0$) on the same split and seed, measured in test weighted F1 on Cora and test AUROC on Reddit and averaged over the ten splits.

\begin{figure}[!t]
\centering
\begin{subfigure}{\textwidth}
  \centering
  \includegraphics[width=\linewidth]{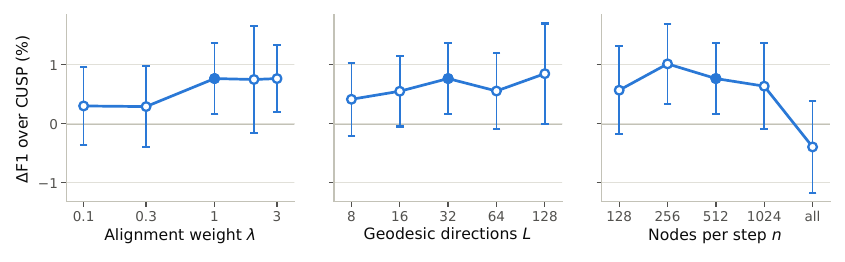}
  \caption{Node classification: CUSP\,+\,CCSGW on Cora.}
  \label{app_fig:sensitivity_cora}
\end{subfigure}
\\[4pt]
\begin{subfigure}{\textwidth}
  \centering
  \includegraphics[width=\linewidth]{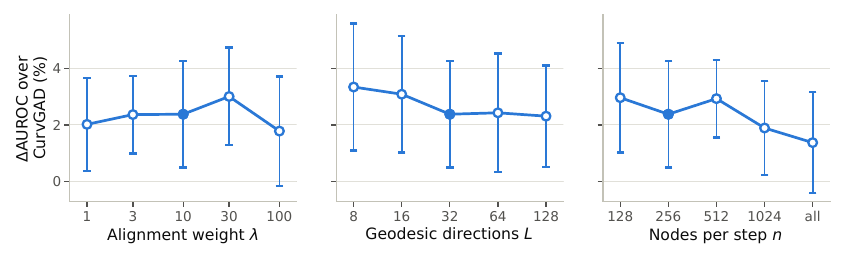}
  \caption{Graph anomaly detection: CurvGAD\,+\,CCSGW on Reddit.}
  \label{app_fig:sensitivity_reddit}
\end{subfigure}
\caption{Hyperparameter sensitivity of CCSGW.
Each panel varies one hyperparameter, namely the alignment weight $\lambda$, the number of geodesic directions $L$ or the number of sampled nodes $n$ (``all'' uses every node), while the other two are fixed at the reference configuration (filled markers).
The vertical axis shows the gain $\Delta$ of the backbone trained with CCSGW over the same backbone trained without it (gray line at zero).
Points and error bars are the mean and 95\% confidence interval over ten splits.}
\label{app_fig:sensitivity}
\end{figure}

CCSGW is robust to all three hyperparameters.
On both datasets, every configuration except one stays within the confidence interval of the reference configuration, and the gain remains positive across two orders of magnitude of $\lambda$ and a sixteenfold range of $L$.
The insensitivity to $L$ follows from redrawing the directions at every epoch, which averages the approximation error of a small $L$ over training.
Subsamples of $128$ to $1024$ nodes perform comparably, whereas using all nodes gives the smallest gain on both datasets and the only negative one, on Cora.
Subsampling therefore reduces the cost of CCSGW without sacrificing accuracy.

\subsection{Comparison with SGW and GW}
\label{app_sec: align_compare}
CCSGW differs from SGW in slicing along geodesics rather than Euclidean directions, and from GW in slicing at all.
To isolate the effect of each choice, we replace CCSGW in the CUSP objective by either loss and keep everything else fixed.
SGW shares the implementation of CCSGW, including the independent directions of every component, with the geodesic projections replaced by linear projections in the ambient coordinates, which corresponds to setting every curvature to zero.
GW is the entropic Gromov--Wasserstein discrepancy between the geodesic distance matrices of the $n$ sampled nodes, normalized by their mean distance and solved with ten projected-gradient steps of fifty Sinkhorn iterations each ($\varepsilon=0.05$).
All three losses use the weight $\lambda=1$, $L=32$ directions where applicable, the node subsample of the main experiments ($n=256$ for the homophilic and $512$ for the heterophilic datasets), and the CUSP configuration and ten splits of \cref{app_exp_node}.

\begin{table}[!ht]
\caption{Comparison of alignment losses on top of CUSP (test weighted F1 \%($\uparrow$) $\pm$ 95\% confidence interval over ten splits).
The last column is the mean gain over CUSP across the five datasets. Best result per column in bold.}
\label{app_tab:align_compare}
\centering
\begin{adjustbox}{max width=\textwidth}
\begin{tabular}{@{}lcccccc@{}}
\toprule
\textbf{Method} & Cora & CiteSeer & PubMed & Chameleon & Actor & \textbf{Av.~$\Delta$} \\ \midrule
CUSP & 85.64\scriptsize{$\pm$1.48} & 74.10\scriptsize{$\pm$0.88} & 90.34\scriptsize{$\pm$0.21} & 66.48\scriptsize{$\pm$2.21} & 38.90\scriptsize{$\pm$0.93} & -- \\ \midrule
CUSP\,+\,CCSGW & \textbf{86.66\scriptsize{$\pm$0.93}} & \textbf{74.24\scriptsize{$\pm$1.23}} & \textbf{90.39\scriptsize{$\pm$0.18}} & 67.82\scriptsize{$\pm$1.92} & 39.83\scriptsize{$\pm$0.70} & +0.70 \\
CUSP\,+\,SGW & 86.19\scriptsize{$\pm$1.00} & 73.62\scriptsize{$\pm$1.22} & 90.32\scriptsize{$\pm$0.14} & 67.49\scriptsize{$\pm$1.23} & 39.23\scriptsize{$\pm$0.84} & +0.28 \\
CUSP\,+\,GW & 86.18\scriptsize{$\pm$1.16} & 73.37\scriptsize{$\pm$1.03} & 90.31\scriptsize{$\pm$0.25} & \textbf{68.48\scriptsize{$\pm$1.10}} & \textbf{39.98\scriptsize{$\pm$0.91}} & +0.57 \\
\bottomrule
\end{tabular}
\end{adjustbox}
\end{table}

\cref{app_tab:align_compare} shows that CCSGW improves over SGW on all five datasets, with a mean gain over CUSP of $+0.70$ against $+0.28$.
Since the two losses differ only in the direction of the slices, most of the benefit comes from slicing along geodesics, which costs the same as slicing along Euclidean directions (\cref{app_sec: cost}).
CCSGW is also on par with GW: it is better on the three homophilic datasets and slightly worse on Chameleon and Actor.
Replacing the full Gromov--Wasserstein discrepancy by its sliced version therefore loses no accuracy, while it reduces the cost from quadratic to log-linear in the number of sampled nodes (\cref{app_sec: cost}). CCSGW thus combines the accuracy of GW with the cost of SGW.

\subsection{Computational cost}
\label{app_sec: cost}
CCSGW and SGW cost $O(L\,n\log n)$ per component pair, whereas entropic GW on geodesic distances requires $O(n^2)$ memory and time per Sinkhorn iteration.
We measure the wall-clock time of one evaluation of each alignment loss, including the backward pass, as a function of the number of sampled nodes $n$, on a single A100 GPU (median of seven repetitions after one warm-up call).
Embeddings are drawn at random in the mixed-curvature space of each backbone, and every loss aligns the same $n$ nodes with $L=32$ directions and standardized slices.
Since CCSGW aligns every pair of components, the number of aligned pairs is three for the CUSP signature $\mathbb{H}^{16}\times\mathbb{S}^{16}\times\mathbb{E}^{16}$ and one for the CurvGAD signature $\mathbb{H}^{24}\times\mathbb{S}^{24}$.

\begin{figure}[!t]
\centering
\begin{subfigure}{0.48\textwidth}
  \centering
  \includegraphics[width=\linewidth]{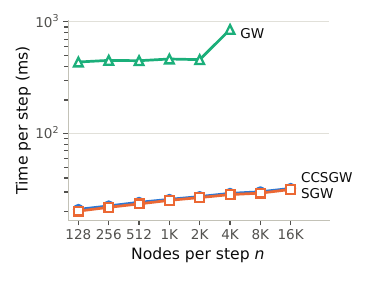}
  \caption{CUSP: three components, three aligned pairs.}
  \label{app_fig:align_timing_cusp}
\end{subfigure}
\hfill
\begin{subfigure}{0.48\textwidth}
  \centering
  \includegraphics[width=\linewidth]{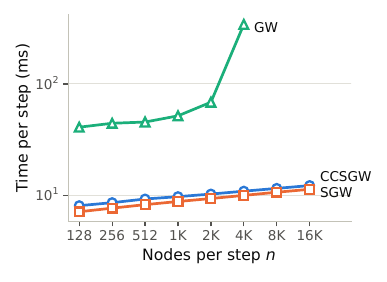}
  \caption{CurvGAD: two components, one aligned pair.}
  \label{app_fig:align_timing_curvgad}
\end{subfigure}
\caption{Time of one alignment-loss evaluation (forward and backward) as a function of the number of sampled nodes $n$, on a single A100 GPU (median of seven runs, log--log axes).
GW runs out of memory beyond $n=4096$.}
\label{app_fig:align_timing}
\end{figure}

\cref{app_fig:align_timing} shows that the geodesic projections come essentially for free: CCSGW and SGW coincide at every $n$, and both remain in the millisecond range over the whole tested range.
GW is one to two orders of magnitude slower from a few thousand nodes on, where its $n^2$ cost matrices dominate, and exhausts the $40$\,GB of memory at $n=8192$.
The higher cost on CUSP reflects the larger number of aligned pairs.
For the subsample sizes used in our experiments ($n\le512$), CCSGW adds only a small fraction of the cost of one training epoch of either backbone.

\end{document}